\documentclass[11pt]{article}
\usepackage[utf8]{inputenc}
\usepackage[margin=1in]{geometry}

\usepackage{amsmath,amsthm,amsfonts,amssymb,mathdots,array,mathrsfs,bm,bbm,stmaryrd,graphicx,subfigure,xcolor}
\usepackage{microtype}
\usepackage{graphicx}
\usepackage{subfigure}
\usepackage{booktabs} 
\usepackage{natbib}
\setcitestyle{open={(},close={)}}
\usepackage[colorlinks=true,citecolor=blue]{hyperref}

\usepackage{multicol}

\usepackage[ruled,vlined]{algorithm2e}

\usepackage{amsmath}
\usepackage{amsfonts}
\usepackage{amssymb}
\usepackage{mathbbol}
\usepackage{graphicx}
\usepackage{fullpage}
\usepackage{url}
\usepackage{csquotes}
\usepackage[american]{babel}
\usepackage{comment}
\usepackage{multirow}
\usepackage{url}
\usepackage{times,subfigure}
\usepackage{physics}

\newtheorem{theorem}{Theorem}
\newtheorem{lemma}[theorem]{Lemma}

\newtheorem{corollary}{Corollary}

\usepackage{xcolor}

\begin{document}

\title{\huge Finding Control Synthesis for Kinematic Shortest Paths}

\author{\vspace{0.5in}\\\textbf{Weifu Wang} \  and \  \textbf{Ping Li} \\\\
Cognitive Computing Lab\\
Baidu Research\\
10900 NE 8th St. Bellevue, WA 98004, USA\\\\
  \texttt{\{harrison.wfw,  pingli98\}@gmail.com}
}

\date{\vspace{0.5in}}
\maketitle

\begin{abstract}\vspace{0.3in}

\noindent This work presents the analysis of the properties of the shortest path control synthesis for the rigid body system. The systems we focus on in this work have only kinematic constraints. However, even for seemingly simple systems and constraints, the shortest paths for generic rigid body systems were only found recently, especially for 3D systems. Based on the Pontraygon's Maximum Principle (MPM) and Lagrange equations, we present the necessary conditions for optimal switches, which form the control synthesis boundaries. We formally show that the shortest path for nearby configurations will have similar adjoint functions and parameters, i.e., Lagrange multipliers. We further show that the gradients of the necessary condition equation can be used to verify whether a configuration is inside a control synthesis region or on the boundary. We present a procedure to find the shortest kinematic paths and control synthesis, using the gradients of the control constraints. Given the shortest path and the corresponding control sequences, the optimal control sequence for nearby configurations can be derived if and only if they belong to the same control synthesis region. The proposed procedure can work for both 2D and 3D rigid body systems. We use a 2D Dubins vehicle system to verify the correctness of the proposed approach. More verifications and experiments will be presented in the extensions of this work. 
\end{abstract}

\newpage

\section{Introduction}

In this work, we present the analysis of the kinematic control synthesis for generalized rigid bodies. We study rigid body systems like Dubins' vehicle, whose controls are selected from a discrete control set rather than a continuous region. Such systems are usually nonholonomic, and their inverse kinematics and the shortest paths are hard to find even without the constraint on the second-order derivatives, i.e., acceleration. 

In the past, there have been works on finding shortest paths for different nonholonomic systems, including Dubins' vehicle~\citep{dubins1957curves}, Reeds-Shepp car~\citep{reeds1990optimal}, Differential Drive~\citep{balkcom2002time}, and Omnidirectional vehicles~\citep{balkcom2006time,wang2012analytical}. It was shown that analytical solutions exist for such systems, but only because they are highly symmetric. Studies on general rigid bodies in 2D were conducted, and a resolution complete algorithm was found~\citep{furtuna2008generalizing,wang2012sampling}. In~\citet{balkcom2018dubins,wang2018time,wang2021towards}, the authors proposed algorithms to find the shortest paths for 3D generic bodies, but the algorithms to find the algorithm is only guaranteed to be resolution complete and usually involve a search over continuous space. The results, though, can be arbitrary close to optimal up to search and representation resolution, the time cost is usually high. 

We can find control synthesis and shortest paths for symmetric systems through analysis or exhaustive search. Since the shortest path can be found analytically for such systems, the exhaustive search for the control synthesis is acceptable. The control synthesis usually shows the optimal control sequences for different starting configurations to the origin. The nearby configurations usually have similar optimal control sequences to the same goal for the known systems. 

There has been little work on control synthesis for general rigid body systems. In this work, we present the analysis of the necessary conditions for configurations being on the boundary of a control synthesis region, i.e., a control switching curve. Gradient-based criteria can test if a configuration satisfies the requirements for being on the switching curve. Building upon the Lagrange equation formulation and the fact that the derived condition is equivalent to the necessary conditions derived from Pontryagin's Maximum Principle (PMP), we analyze how to generate control synthesis for general rigid bodies with discrete control sets. 

The main contribution of this work is a verification criterion for a configuration being on the boundary of two control synthesis regions. The method can be adjusted to test if a configuration is inside a control synthesis region. The same test can find the direction to adjust the adjoint function parameters to approach the value corresponding to the shortest path at a given configuration. We present a new approach to finding the shortest path for a given pair of start and goal by combining the control synthesis test criteria. 

This work only considers the rigid body systems with a discrete set of controls rather than a continuous set of controls. However, the analysis for symmetric systems has shown that the {\em extreme controls} on the vertices of control regions are the candidate optimal controls to achieve the shortest paths. However, we have not been able to prove similar statements for general rigid bodies. Further, we do not consider the case where the shortest path does not exist, i.e., chattering behavior that can consistently create shorter paths, such as Dubins' vehicle without translation control. 

\newpage

\section{Related Works}

The shortest path problem has been studied well in-depth, but most approaches are either using Hamilton-Jacobi-Bell (HJB) equations~\citep{bellman1965dynamic} or Pontryagin's Maximum Principle (PMP)~\citep{boltyanskiy1962mathematical}. The HJB equations provide a sufficient condition for optimality and are suitable for numerical methods to approximate the optimal solutions, especially for complex systems. The Maximum Principle (PMP), on the other hand, only provides a necessary condition for optimality. The necessary conditions are strong and can sometimes lead to analytical solutions for some systems. 

For specific systems, shortest paths and control synthesis have been found through mostly geometrical analysis. One of the earliest works on the shortest paths for jet airplanes cruising at the same altitude is now known as a Dubins' vehicle model~\citep{dubins1957curves}. The shortest geodesics are proved to consist of only line segments and arcs with the same radii. Control synthesis for the Dubins' model was presented in~\citet{bui1994shortest}. The work of~\citet{reeds1990optimal} further extended the model to include backward controls, making the system resembles a steered car. Further analysis of this model was presented in~\citet{sussmann1991shortest,boissonnat1994shortest}. Control synthesis for this model was later presented by Soueres and Lamound~\citep{soueres1996shortest}. Additional models, such as Differential Drive~\citep{balkcom2002time,chitsaz2009minimum} and Omni-directional vehicles~\citep{balkcom2006time,wang2012analytical} were studied. 

Building on the results for the above systems and the methods used, Furtuna {\em et al.} started to extend the analysis to generic rigid bodies in 2D~\citep{furtuna2008generalizing,furtuna2011minimum,furtuna2011thesis}. A resolution complete algorithm was developed in~\citet{wang2012sampling} to find the shortest paths for 2D rigid bodies. Recently, the studies were further extended to 3D. Through the comparison between rigid body systems to robotic arms model, 3D analysis was built upon the Lagrange systems rather than PMP, as derivations are much easier compared to integration~\citep{balkcom2018dubins}. Resolution complete algorithms were developed for 3D systems with only positional constraints~\citep{wang2018time} and positional and orientation constraints~\citep{wang2021towards}. 

3D shortest paths are much more complex to find than 2D systems, as there are three orientation parameters. \citet{chitsaz2007time} studied a simplified model of Dubins' vehicle with altitude control. In~\citet{yang2002optimal}, the optimal control problem was studied. 

Chattering behavior may appear when the acceleration has no bounds, i.e., only considering kinematic constraints. When chattering can happen, the shortest path may not exist as more chattering may lead to shorter paths, such as a Dubins' vehicle without translation. For some systems, it can be proved that there always exist non-chattering paths to guarantee the existence of the shortest paths~\citep{furtuna2011thesis}. In~\citet{lyu2014bench,lyu2016optimal}, the authors showed that by adding a cost between switches and applying Blatt's Indifference Principle (BIP)~\citep{blatt1976optimal}, the solutions are similar to those trajectories with bounded acceleration and can prevent the chattering behavior to some degree. 

When acceleration bounds do exist, the above solutions no longer apply. The acceleration constraint and dynamic effects can alter the control strategy. It was proved that no analytical solutions of time-optimal trajectories exist for bounded-acceleration vehicles~\citep{sussmann1997markov,soueres1998optimal}. Variations of the model have been introduced in the past, including underwater vehicles~\citep{chyba2005designing}, and under a constant velocity field~\citep{dolinskaya2012time}.

All the above analysis only applies when no obstacle is present. Though PMP gives conditions to consider obstacles in the formulation, the integration becomes even more complex. There have been works on measuring distances between a car-like robot with obstacles~\citep{vendittelli1999obstacle,giordano2009shortest}, and planning among simple obstacles for car-like systems~\citep{gesaulniers1996shortest,agarwal1997nonholonomic}.

\newpage

\section{Finding Shortest paths Using Lagrange Multipliers}

We briefly restate the results from~\citet{wang2021towards}, which is the basis of the following analysis. The configuration of a mobile robot after a sequence of {\em controls} can be computed using forward kinematics, just like how we compute the configuration of the end effector of a robot arm. By control, we mean a constant velocity motion, either a translation or rotation for simplicity. Such simplification of controls enables the analysis of paths with geometrical tools, leading to the shortest geodesics or time-optimal trajectories referred to by some literature. 

Given a start configuration $q_s$ and a goal configuration $q_g$, and a control set $U = \{u_1, u_2, \ldots, u_m\}$ all with constant velocity, a shortest path from $q_s$ to $q_g$ is a trajectory with $n$ segments. The robot performs control $u_j\in U$ on the $i$th segment for duration $t_i$ so that the summation of all $t_i$ are minimized among all possible control sequences. 
\begin{eqnarray}
\min f(t) = \sum_{i = 1}^n t_i\\
s.t. \ \ q_s + T_f = q_g\label{eq:constraint_reach_goal}
\end{eqnarray}

In Eq.~\eqref{eq:constraint_reach_goal}, term $T_f$ is the result of a sequence of matrix ($T_i$) multiplication in a given order. Each matrix $T_i$ is a transformation matrix, corresponding to the transformation of the robot by applying a control $u_\cdot$ for duration $t_i$. More details of the derivation of $T_f$ can be found in~\citet{balkcom2018dubins}. Let us denote $h(t) = T_f - \hat{g}$, using the method of Lagrange multipliers, we want to find a vector $\mathbf t$ such that
\begin{align}
\grad_t f(t) = \lambda \grad_t h(t),
\end{align}
at points where $h(t) = 0$. 

The gradient of the objective function $\grad_t f(t)$ is a $n$ vector of ones. Writing out the Lagrange condition in the matrix form and putting each constraint in a column, we can find that the gradient of the constraint is in the same form as an analytical Jacobian for a serial arm. The analytical Jacobian is identical to its geometric Jacobian for Cartesian coordinates. The geometric Jacobian can be computed using the cross-product rule. A translation control's geometric Jacobian is the linear velocity vector. A rotation control's geometric Jacobian can be computed as a cross-product between the rotation axis and the vector from the rotation center to the goal $g$. 

We have $\lambda v_i = 1$ for a translation control at segment $i$, and $\lambda\cdot(\hat\omega_i\times\vec{r_ig}) = 1$, where $\hat\omega_i$ is the rotation axis for the $i$th control, and $r_i$ is the corresponding rotation center. In 2D, all rotation axes are parallel to $z$-axis, we can convert $\hat\omega_i$ into $(0, 0, 1)^T\cdot\omega_i$, where $\omega_i\in\mathbb{R}$ is a signed value, corresponding to clockwise or counter-clockwise rotation respectively. Then, we can simplify the equation into the following form in 2D,
\begin{eqnarray}
\lambda_1\dot{x} + \lambda_2\dot{y} + \omega_i(\lambda_1 y-\lambda_2 x + c) = 1
\label{eq:2d_generic}
\end{eqnarray}
where $c$ is a constant computed from $\lambda$ and $q_g$. This equation is identical to that derived using Pontraygon's Maximum Principle (PMP)~\citep{boltyanskiy1962mathematical,furtuna2011thesis}. In~\citet{wang2021towards}, the authors showed that in 2D, the derived necessary condition using Lagrange equation is equivalent to that derived using PMP. 

Because the geometric Jacobian is much easier to compute than the analytical Jacobian for Cartesian points, we convert the representation of the configuration from the minimum representation with orientations to a collection of points with fixed distances. Such a change in representation is especially beneficial for 3D systems. In~\citet{wang2021towards}, the authors demonstrated the equivalence between the minimum and the point-based representation (redundant). 

Let us denote the configuration of a 3D rigid body as a collection of three points, $q = (p_o, p_x, p_y)$, where $p_i = (x, y, z)$ where $i = \{o, x, y\}$. For simplicity, let $p_o$, $p_x$ and $p_y$ be the origin $(0, 0, 0)$, $(1, 0, 0)$, and $(0, 1, 0)$ from the robot frame. Then, the problem of finding optimal trajectory becomes finding the path that can lead all three points to their respected goal positions simultaneously. As the three points are fixed in robot frame, their distances are maintained throughout the trajectory. 

Given a sequence of $n$ controls, we need the three points to reach the goal simultaneously. Using forward kinematics, we can find matrices $T_f^o, T_f^x, T_f^y$, and find the shortest-path using the following constrained optimization: 
\begin{eqnarray}
\min f(t) = \sum_{i = 1}^n t_i\\
s.t. \ \ h_1(t) = T_f^o - g_o = 0\\
h_x(t) = T_f^x - g_x = 0\\
h_y(t) = T_f^y - g_y = 0
\end{eqnarray}

Using the Lagrange multipliers, we can have $\grad_t f(t) = \lambda_o \grad_t h_1(t) + \lambda_x \grad_t h_2(t) + \lambda_y \grad_t h_3(t)$. Denote $\lambda_o = (\lambda_1, \lambda_2, \lambda_3)$, $\lambda_x = (\lambda_4, \lambda_5, \lambda_6)$, and $\lambda_y = (\lambda_7, \lambda_8, \lambda_9)$, we can rewrite the optimal condition as
\begin{eqnarray}
(\lambda_1, \lambda_2, \lambda_3)^T\cdot(\omega\cdot\hat{w}_i\times \vec{r_ig_o}) +\\
(\lambda_4, \lambda_5, \lambda_6)^T\cdot(\omega\cdot\hat{w}_i\times \vec{r_ig_x}) +\\
(\lambda_7, \lambda_8, \lambda_9)^T\cdot(\omega\cdot\hat{w}_i\times \vec{r_ig_y}) = 1
\end{eqnarray}
when the control $i$ is a rotation around $\hat{w}_i$ at angular velocity $\omega$, or,
\begin{eqnarray}
(\lambda_1, \lambda_2, \lambda_3)^T\cdot(\dot x(t), \dot y(t), \dot z(t)) +\\
(\lambda_4, \lambda_5, \lambda_6)^T\cdot(\dot x(t), \dot y(t), \dot z(t)) +\\
(\lambda_7, \lambda_8, \lambda_9)^T\cdot(\dot x(t), \dot y(t), \dot z(t)) = 1
\end{eqnarray}
when the control $i$ is a translation with linear velocity $v = (\dot x(t), \dot y(t), \dot z(t))$. For simplicity, in the text below, let $\hat{\omega}_i = \omega\hat{w}_i$. 

We can rewrite the above equations by combining different terms together, replacing $\hat{\omega}_i\times \vec{r_ig}$ with $\hat{\omega}_i\times (\vec{r_ip_o} + \vec{g} - \vec{p_o})$, making the first part of the cross-product equivalent to the linear velocity. Thus, we have
\begin{eqnarray}
\vec{\lambda}(\dot x(t), \dot y(t), \dot z(t)) + \hat{\omega}_i(\vec{\lambda}\times\vec{p} + \vec{g_o}\times\vec{\lambda}+\vec{c^\alpha}) = 1
\label{eq:lag_3d}
\end{eqnarray}
where $\vec{\lambda} = (\lambda_1, \lambda_2, \lambda_3) + (\lambda_4, \lambda_5, \lambda_6) + (\lambda_7, \lambda_8, \lambda_9)$, and $\vec{c^\alpha} = \vec{g_og_x}\times(\lambda_4, \lambda_5, \lambda_6)^T + \vec{g_og_y}\times(\lambda_7, \lambda_8, \lambda_9)^T$. As all $\lambda$ are constant Lagrange multipliers, and the goal positions are fixed, $\vec{c}$ is a constant vector for a given goal configuration. We omit the details, but it is easy to show that if the system state are described as $q = (p_o, d_x, d_y)$ where $d_x = \vec{p_op_x}$ and $d_y = \vec{p_op_y}$, the resulting condition can be reorganized into the same form. Let $\vec\lambda = \alpha\cdot\vec k$, and $\vec{c^\alpha} = \alpha\cdot\vec{c}$, we then have
\begin{eqnarray}
\vec k\cdot(\dot x(t), \dot y(t), \dot z(t)) + \hat{\omega}_i\cdot(\vec{p_og}\times\vec{k} + \vec{c}) = H\label{eq:nc_generic_rewrite}\\
\vec k\cdot(\hat{\omega}_i\times\vec{r_ig}) + \hat{\omega}_i\cdot\vec{c} = H \label{eq:nc_generic}
\end{eqnarray}

Eq.~\eqref{eq:lag_3d} through Eq.~\eqref{eq:nc_generic} are equivalent. There are two parts in Eq.~\eqref{eq:nc_generic}, one is the dot product between $\vec{k}$ and control moment, which is $\hat{\omega}_i\times \vec{r_ig}$ for rotation and $(\dot{x}, \dot{y}, \dot{z})^T$ for translation. This control moment is the geometric Jacobian for Cartesian components with respect to rotation and translation. There is another term, dot product between rotation axis $\hat{\omega}_i$ and a constant vector $\vec{c}$. Note that the geometric Jacobian for an orientation is the corresponding rotation axis in world frame. 

The time-optimal necessary condition can be interpreted as the two components of the geometric Jacobian dot product with two constant vectors, so the result must remain constant along the trajectory. The condition is interesting in the construction and study of more than six degrees of freedom of robot arms, where both constant vectors would be fixed and unique for any given configuration the arm attempts to reach. 

The derived necessary condition in Eq.~\eqref{eq:nc_generic_rewrite} is an extension of the Eq.~\eqref{eq:2d_generic} in 2D. The dot product with velocity of the system appears in both equations, and the remaining term is governed by rotational components. In 2D, all rotation axis are along $z$-axis, simplifying the Eq.~\eqref{eq:nc_generic} into $(\lambda_o+\lambda_x)\cdot(\dot{x_o}, \dot{y_o})^T + \omega(g\times\lambda_o + g_x\times\lambda_x - p_o\times(\lambda_o+\lambda_x)) = 1$, which can be written as  $\lambda\cdot(\dot{x_o}, \dot{y_o})^T + \omega(\lambda\times p_o + c)$ where $\lambda = \lambda_o + \lambda_x$. This is the same condition derived from PMP in the previous work~\citep{wang2012sampling}. However, in 3D, as the rotational axes no longer point along the same direction, the above stated constants cannot be obtained. Even though the forms of two equations are the same, and would produce same geometric operations, the results can be quite different. 

\section{Rigid body shortest path requires a three-dimensional search}

From Eq.~\eqref{eq:nc_generic_rewrite}, we can see that the necessary condition for shortest path has at least six unknown parameters, including $\vec{k}$ and $\vec{c}$. In~\citet{wang2021towards}, we showed that the search could be reduced to three dimensions to only search for $\vec{k}$. The main argument was that the shortest path to reach a goal must also satisfy the necessary condition for only reaching the goal in position. Therefore, there must exist a $\vec{k}'$ that satisfy the relation $\vec{k}' (\hat\omega_i\times\vec{r_ig}) = H$ for rotation controls. 

In this work, we first demonstrate that the derivation from PMP also proves a three dimensional search is sufficient to find the shortest path. First, let us briefly restate the PMP conditions. We need to find the representation for $H(t)$, which is the dot product between the adjoint function $\lambda(t)$ and the velocity of the system $\dot{q}(q, u)$. The velocity depends on the configuration and the control of the robot. Using the similar configuration representation, we have $\dot{q} = (\dot{p_o}, \dot{p_x}, \dot{p_y})$, which is the transformation of the control in robot frame into the world frame. In the robot frame, we have $u(t) = (\dot{p_o^R}, \dot{p_x^R}, \dot{p_y^R})$, which is the velocity of the corresponding three points in the robot frame, and denote the transformation matrix of $u(t)$ to $\dot{q}(t)$ as $\mathcal{R}$, which is a tridiagonal matrix of dimension nine by nine. The tridiagonal elements are the displacement of the same three by three matrix $R$, which can transform a vector in the robot frame at time $t$ to the world frame. The matrix $R$ can be written as,
\begin{eqnarray}
R = \begin{bmatrix}
d_x^T & d_y^T & d_z^T
\end{bmatrix}
\end{eqnarray}
where $d_z = d_x\times d_y$. Now, if we take derivatives of $H(t)$ with respect to $q = (x, y, z, \theta, \phi, \psi)$, the configuration, we will have non-zero matrix for all $\frac{\partial H}{\partial q_i}$. It is possible that we can use another notation for the configuration, $q^\Delta = (p_o, d_x, d_y)$, where $d_x = \vec{p_op_x}$ and $d_y = \vec{p_op_y}$ are defined the same as above. Then, we have the derivatives of the adjoint function as follows: 
\begin{eqnarray}
\frac{d\lambda}{dt} &=& -\frac{\partial H}{\partial q}\\
\Rightarrow \dot{\lambda} &=& - \frac{\partial}{\partial q}\langle\lambda, \dot{q}(q, u)\rangle
\end{eqnarray}

The first three terms are $0$ because the matrix $\mathcal{R}$ is independent of $p_o$, but depends on the $d_x$ and $d_y$. Therefore, we know that the first three dimensions of the adjoint function $\lambda$ can be constants, agree with the results derived from Lagrange multipliers. 

The remaining dimensions of the adjoint function, however, are not so simple. The derivatives of the rotation matrix with respect to different components of the $d_x$ and $d_y$ vectors can be derived as follows, 
\begin{eqnarray}
\frac{\partial}{\partial d^x_x}R = \begin{pmatrix}
1 & 0 & 0 \\ 0 & 0 & -d^z_y \\
0 & 0 & d^y_y
\end{pmatrix}, \ \frac{\partial}{\partial d^y_x}R = \begin{pmatrix}
0 & 0 & d_y^z \\ 1 & 0 & 0 \\
0 & 0 & -d^x_y
\end{pmatrix}, \ \frac{\partial}{\partial d^z_x}R = \begin{pmatrix}
0 & 0 & -d_y^y \\ 0 & 0 & d^y_x \\
1 & 0 & 0
\end{pmatrix}\\
\frac{\partial}{\partial d^x_y}R = \begin{pmatrix}
0 & 1 & 0 \\ 0 & 0 & d^z_x \\
0 & 0 & -d^y_x
\end{pmatrix}, \ \frac{\partial}{\partial d^y_y}R = \begin{pmatrix}
0 & 0 & -d_x^z \\ 0 & 1 & 0 \\
0 & 0 & d^x_x
\end{pmatrix}, \ \frac{\partial}{\partial d^z_y}R = \begin{pmatrix}
0 & 0 & d_x^y \\ 0 & 0 & -d^x_x \\
0 & 1 & 0
\end{pmatrix}
\end{eqnarray}

Following the above equations, the $\dot\lambda$ vectors are: 
\begin{eqnarray}
\dot\lambda_1 = 0, \ \dot\lambda_2 = 0, \ \dot\lambda_3 = 0\\
\dot{\lambda_4} = \lambda_1 u(p_o^x) + \lambda_2 (-d^y_z u(p_o^z)) + \lambda_3 d^y_y u(p_o^z) +  \nonumber\\ 
\lambda_4 u(d^x_x) + \lambda_5 (-d^x_y u(d^z_x)) + \lambda_6 d^y_y u(d^z_x) + \nonumber\\
\lambda_7 u(d^x_y) + \lambda_8 (-d^x_y u(d^y_y)) + \lambda_9 d^y_y u(d^z_y))
\label{eq:dxx}\\
\dot\lambda_5 = \lambda_1 d_y^z u(p_o^z) + \lambda_2 u(p_o^x) + \lambda_3(-d_y^x u(p_o^z)) +\nonumber\\
\lambda_4 d_y^z u(d_x^z) + \lambda_5 u(d_x^x) + \lambda_6 (-d_y^x u(d_x^z)) + \nonumber\\
\lambda_7 d_y^z u(d_y^z) + \lambda_8 u(d_y^x) + \lambda_9 (-d_y^x u(d_y^z)\\
\dot\lambda_6 = \lambda_1(-d_y^y u(p_o^z)) + \lambda_2 d_y^x u(p_o^z) + \lambda_3 u(p_o^x) + \nonumber\\
\lambda_4 (-d_y^y u(d_x^z)) + \lambda_5 d_y^x u(d_x^z) + \lambda_6 u(d_x^x) + \nonumber\\
\lambda_7 (-d_y^y u(d_y^z) + \lambda_8 d_y^x u(d_y^z) + \lambda_9 u(d_y^x) \\
\dot\lambda_7 = \lambda_1 u(p_x^y) + \lambda_2 d_x^z u(p_o^z) + \lambda_3 (-d_x^y u(p_o^z)) + \nonumber \\
\lambda_4 u(d_x^y) + \lambda_5 d_x^z u(d_x^z) + \lambda_6 (-d_x^y u(d_x^z)) + \nonumber \\
\lambda_7 u(d_y^y) + \lambda_8 d_x^z u(d_y^z) + \lambda_9 (-d_x^y u(d_y^z)) \\
\dot\lambda_8 = \lambda_1 (-d_x^z u(p_o^z)) + \lambda_2  u(p_o^y) + \lambda_3 d_x^x u(p_o^z) + \nonumber \\
\lambda_4 (-d_x^z u(d_x^z)) + \lambda_5  u(d_x^y) + \lambda_6 d_x^x u(d_x^z) + \nonumber \\
\lambda_7 (-d_x^z u(d_y^z)) + \lambda_8 u(d_y^y) + \lambda_9 d_x^x u(d_y^z)\\
\dot\lambda_9 = \lambda_1 d_x^y u(p_o^z) + \lambda_2  (-d_x^x u(p_x^z)) + \lambda_3 u(p_o^y) + \nonumber \\
\lambda_4 d_x^y u(d_x^z) + \lambda_5  (-d_x^x u(d_x^z)) + \lambda_6 u(d_x^y) + \nonumber \\
\lambda_7 d_x^y u(d_y^z) + \lambda_8 (-d_x^x u(d_y^z)) + \lambda_9 u(d_y^z)
\end{eqnarray}
where $u(d^x_x)$ is the $x$ component of $u(d_x)$. The above equations show that the derivatives of the remaining dimensions of the adjoint functions are not an obvious $0$. However, we know the two systems derive the same necessary condition. All the $\dot\lambda$ must equal $0$ so that the $\lambda$ can be a constant. Having all the $\dot\lambda$ terms being $0$ is a feasible condition to satisfy. Essentially, we are enforcing the intersection of the null space of a matrix function $A(t)$ to be non-empty for all candidates of optimal controls and configurations on the optimal trajectory. We know there are control limitations on what configurations can be optimal. Therefore, we need to find the set of controls at given configurations that admit the above condition. 

Having all $\dot\lambda = 0$ give us six equations. Since $\lambda$ contains nine unknowns, having six equations left us with three unknown variables to specify. The six equations mean that finding the shortest path for any given pair of start and goal requires searching for three unknown variables, i.e., a three-dimensional search. In Eq.~\eqref{eq:nc_generic_rewrite}, constant $\vec{c}$ is a combination of goal configurations and $\lambda$ vector. Knowing the goal, we have non-independent $\vec{c}$ and $\vec{k}$. The sufficiency of only searching for $\vec{k}$ to find the shortest path is critical in our following analysis of control synthesis. 

We can have constant Lagrange multipliers in such a system because we select a discrete set of controls without acceleration bound, i.e., the velocity is the control for the given systems. The system is not a second-order differential equation system. When the acceleration (second-order derivatives) are involved, the adjoint function $\lambda(t)$ does not remain constant along the entire trajectory. The Lagrange multiplier corresponds to the velocity constraints, while the adjoint function $\lambda$ regulates the acceleration controls.


\section{Shortest Path Control Synthesis}

Now, let us formally describe the control synthesis of the shortest paths for rigid body systems, which is the main focus of this work. In particular, let us consider the goal at origin and all the starting configurations with the same orientation. Let us define a control synthesis region as all the configurations (with the same orientation) whose shortest path to the goal (origin) shares the same control sequence. Let us refer to a switching curve as a set of connected configurations on the boundary of two adjacent control synthesis regions.  

Let us first analyze the switching curve between two translation controls. Both translation controls must maintain the same dot product with $\vec{k}$. Since we are considering only the configurations with the same orientation, the configurations on this curve must have corresponding $k$ vectors with the same direction. If there are two or more translation controls for a given system, two translation controls can switch at any time without violating the necessary condition for optimality. The system can even chatter between two translation controls. Therefore, there does not exist a switching curve between two translation controls. We do not consider the switch between two translations in the remaining sections. 

In other cases, a configuration on a switching curve between two controls $u_i$ and $u_j$ will require the two controls both maintain the constant dot product with the same $\vec{k}$ and $\vec{c}$. As the $\vec{k}$ can be scaled to produce different $H$ values, we will choose $\vec{k}$ to have constant dot product of $1$ with control moments. By control moment, we mean the linear velocity vector for translation controls, and the cross-product between the rotation axis and the vector from rotation center to the goal for rotation controls.  We have,
\begin{eqnarray}
\vec{k}\cdot v_i = 1\ (Translation)\\
\vec{k}\cdot (\hat\omega_i\times\vec{r_ig}) + \omega_i\cdot\vec{c} = 1 \ (Rotation)
\end{eqnarray}

In 3D, skew motions exist, but a skew motion can be decomposed into a translation and a rotation. We do not consider skew motions here but will analyze them in future work.

Finding shortest paths and the control synthesis relies on knowing the corresponding $\vec{k}$ and $\vec{c}$. For any given pair of start and goals, a $\hat{k}$ and a $\hat{c}$ exist that regulate the shortest path. If we consider the shortest paths from all configurations to the origin, will each configuration have a unique $\vec{k}$ and $\vec{c}$ corresponding to the shortest path? 

\newpage

\begin{lemma}
Consider all configurations $Q = \{q = (x, y, \theta) | \|x^2+y^2\| = D\}$, the shortest path from any configuration with $\|x^2+y^2\| < D$ to the origin $O$ is part of a shortest path from $q\in Q$ to $O$. The same statement holds for 3D configurations as well. 
\end{lemma}

\begin{proof}
This lemma is an extension of Bellman's condition for optimality. 
\end{proof}

\begin{corollary}
If vectors $\hat{k}$ and $\hat{c}$ are the parameters for the shortest path from configuration $q$ to the origin, they must also be the parameters for the shortest path from other configurations to the origin. 
\end{corollary}

Therefore, not all configurations have unique constant vector $\vec{k}$ and $\vec{c}$. However, if the two configurations belong to the same control synthesis region, the corresponding $\vec{k}$ and $\vec{c}$ are similar. Let us denote a configuration $q$ as $(p, \Omega)$, where $p$ denotes the position component, and $\Omega$ denotes the orientation component.

\begin{lemma}
Given a configuration $q = (p, \Omega)$, and a nearby configuration with the same orientation $q_\Delta = (p_\Delta, \Omega)$, if two configurations belong to the same control synthesis region, their shortest path to the origin share similar $k$ vector. 
\label{lemma:similar}
\end{lemma}

\begin{proof}
Because two configurations are in the same control synthesis region, they share the same first control. Let us first consider the case where the control is rotation. Then, we have, 
\begin{eqnarray}
\vec{k}(\hat\omega_i\times\vec{r_ig}) + \hat\omega_i\cdot\vec{c} = \hat{H}\\
\vec{k}_{\Delta}(\hat\omega_i\times(\vec{r_ig} - \Delta)) + \hat\omega_i\cdot\vec{c}_\Delta = \hat{H}_\Delta 
\end{eqnarray}
Here, we have $\Delta = p_\Delta - p$. Similarly, the corresponding rotation center location is changed by $\Delta$. Decomposing the above equations and let $\vec{k}_\Delta = \vec{k} + \Delta k$. Let us first only consider $\vec{k}$, we can find 
\begin{eqnarray}
\Delta k(\hat\omega_i\times\vec{r_ig}) - \Delta\cdot (\vec{k}_\Delta\times\hat\omega_i) = H_\Delta - H
\end{eqnarray}
Because the two configurations belong to the same control synthesis region, they share the same last control. Whether the last control is rotation or translation, they have the same orientation, or they have the same rotation axis and rotation center. Therefore, the contribution coming from the last control form a constant vector, let us denote the vector $\vec{d}$, we have
\begin{eqnarray}
\vec{k}\cdot\vec{d} = H\\
\vec{k}_\Delta\cdot\vec{d} = H_\Delta\\
\Delta k(\hat\omega_i\times\vec{r_ig}) - \Delta\cdot (\vec{k}_\Delta\times\hat\omega_i) = \Delta k\cdot\vec{d}\\
\Rightarrow \Delta k(\hat\omega_i\times\vec{r_ig}-\vec{d}) = \Delta (\vec{k}_\Delta\times\omega)
\end{eqnarray}

On the right-hand side of the equation, we know $\Delta$ is small, so the resulting dot product must be small. On the left-hand side, to produce a small dot product, either the two vectors are close to perpendicular, or one is very small. Because the first and the last control are different, $\hat\omega_i\times\vec{r_ig}$ and $\vec{d}$ are not similar. Their difference is also perpendicular to $k$. Therefore, $\hat\omega_i\times\vec{r_ig}-\vec{d}$ cannot be perpendicular to $\Delta k$. Therefore, $\Delta k$ is small. 

If the first control is translation, we have 
\begin{eqnarray}
\Delta k(v_i - \vec{d}) = 0
\end{eqnarray}
If $v_i$ and $\vec{d}$ is not the same, $v_i - \vec{d}$ is also perpendicular to $\vec{k}$, this would lead to $\Delta k$ also being small. 

Bringing the similar $\vec{k}$ back to the equations with $\vec{c}$, we can derive similar relations and show that $\vec{c}$ is similar for nearby configurations in the same control synthesis region. 
\end{proof}

\begin{lemma}
Given a configuration $q = (p, \Omega)$, and a nearby configuration $q_\Delta = (p_\Delta, \Omega_\Delta)$ that can be connected by the same control, if two configurations belong to the same control synthesis region, they share similar $k$ vector. 
\end{lemma}

\begin{proof}
If $q$ and $q_\Delta$ are connected by the same control and belong to the same control synthesis region, they must be connected by a rotation. Then, we have the two control centers and control axes being very close. They also share the same last control. Therefore, we can derive similar equations like those in the proof for Lemma~\ref{lemma:similar}. Therefore, $\vec{k}$ and $\vec{c}$ are similar. 
\end{proof}

\begin{theorem}
Two nearby configurations share a similar $k$ vector if the two configurations belong to the same control synthesis region. 
\end{theorem}

\begin{proof}
The theorem follows from the previous two Lemmas. 
\end{proof}

Though it can be shown that both $\vec{k}$ and $\vec{c}$ are similar for nearby configuration on the same slice, it was shown in the previous section that it is sufficient just to consider $\vec{k}$. For simplicity, we will omit the $\vec{c}$ term in the remaining sections. 


We consider control synthesis for configurations with the same orientations, i.e., the same configuration slice, to the goal at the origin. Let there be a start $q_s = (s_o, s_x, s_y)$ on the control synthesis switching curve, which means there are two controls dot products with the {\em control moment} to the same $H$. For a nearby configuration $q_s + \Delta s = (s_o + \Delta_x s, s_x + \Delta_y s, s_y + \Delta_z s)$, if $\Delta s$ is along the gradient of the switching curve at $q_s$, and if $\|\Delta s\|$ is small enough, the $q_s + \Delta s$ should also on the switching curve or at least arbitrarily close to it. The gradient of the switching curve and the gradient of the constraints with respect to the configuration are related. First, we look at the gradient form for a translation control constraint. 
\begin{eqnarray}
k\cdot v_i = 1\\
h(q): k\cdot v_i - 1 = 0\\
\grad h(q) = 0
\end{eqnarray}

In other words, the translation control does not depend on the location of the configuration. The gradient is degenerate. A degenerate gradient means that walking in any direction from the original configuration and the new configuration does not necessarily violate the optimal constraint for translation control. The translation on the shortest path is only constrained by the orientation and is independent of the position. On the same configuration slice, any nearby configuration can potentially have the shortest path starting with the same translation. However, through analysis, we know that only configurations with $\Delta s$ parallel to the linear velocity can potentially remain on this switching curve.

\begin{eqnarray}
h(q) : k\cdot(\hat\omega\times\vec{rg}) - 1 = 0\\
r = q + R(q)\cdot r^R\\
h(q) : k\cdot(\hat\omega\times(g - q - R(q)\cdot r^R)) - 1 = 0\\
\Rightarrow \grad_q h(q) = \begin{pmatrix}
-k_2\omega_z + k_3\omega_y \\
k_1\omega_z - k_3\omega_x \\
-k_1\omega_y + k_2\omega_x
\end{pmatrix} = \omega\times\vec{k}\label{eq:grad_q}
\end{eqnarray}

Above, we have the gradients of the rotation control constraint. Here, we use $\hat{g} = g - R(q) r^R$. The gradient is not trivial. If a configuration is on the switching curve, the new configuration must still satisfy both controls' constraints and share the new $k$ vector. When considering the same orientation slice of the configuration space, rotation axes $\omega$ will not change direction. 

However, the gradient of the constraints with respect to the position is not the tangent of the switching curve. The gradient for each control is the best direction to increase the dot product, rather than maintaining the constraints. 
Let us consider a change of position of $\Delta$ on the configuration slice. We have, 
\begin{eqnarray}
k\cdot(\hat\omega\times\vec{rg}) = 1\\
k'\cdot(\hat\omega\times\vec{r'g}) = 1\\
r'-r = \Delta, k'-k = \Delta k\\
\Delta ((k+\Delta k)\times\hat\omega) = \Delta k(\hat\omega\times\vec{rg})\label{eq:delta_relation}
\end{eqnarray}

This would also require the $\Delta k$ vector to verify whether there exist $\Delta$ direction to satisfy the condition. We take the gradient of the constraint with respect to $\vec{k}$, and have, 
\begin{eqnarray}
\grad_k h(q) = \begin{bmatrix}
\omega_y(\hat{g}_z-z) - \omega_z(\hat{g}_y-y) \\ 
\omega_z(\hat{g}_x-x) - \omega_x(\hat{g}_z-z)\\
\omega_x(\hat{g}_y-y) - \omega_y(\hat{g}_x-x)
\end{bmatrix} = \omega\times\vec{rg}\label{eq:grad_k}
\end{eqnarray}
The change of $k$ is along the direction of the control moment. We can then take the $\grad_k h(q)$ for both switching controls. Different controls will have different $\grad_k h(q)$. However, if a configuration is on the switching curve between two controls, then there does not exist a $\grad_k$ that can simultaneously increase the dot product with the control moment for all three controls, following the maximization condition from PMP. Therefore, we first present the {\em Necessary Condition} for a configuration being on a switching curve: there does not exist $\Delta k$ not parallel to $k$ that can increase the dot product with all control moments. 

The $\grad_k$ is the direction to increase the dot product with the control moment. Regardless of the change direction of $\Delta$, there exist linear combination of $\grad_k$ vectors corresponding to the new $\vec{k}$. On the other hand, the $\grad_q$ of the constraint is the change direction of the rotation center to increase the dot product with the current $\vec{k}$. We need to find the direction for which the change of the control moment for both switching controls will lead to vectors that are possible to maintain a constant dot product with a new $\vec{k}$.  Eq.~\eqref{eq:delta_relation} is stating this constraint for a single control. If a $\Delta k$ and $\Delta$ can satisfy Eq.~\eqref{eq:delta_relation} for both switching controls, the switching curve gradient is found. Jointly solving both equations, we have
\begin{eqnarray}
\Delta\cdot\left[ (k+\Delta k)\times(\hat\omega_i-\hat\omega_j)\right] = \Delta k(\hat\omega_i\times\vec{r_ig} - \hat\omega_j\times\vec{r_jg})
\end{eqnarray}
We have $\hat\omega_i-\hat\omega_j$ being a constant equal to the vector between them in the robot frame. On the right-hand side, values in the parenthesis are also known at the current configuration. Simplifying the equation, we have
\begin{eqnarray}
\Delta\cdot\left[(k+\Delta k)\times\hat\omega_{i, j}\right] = \Delta k\cdot\vec{d}\label{eq:joint_deltas}
\end{eqnarray}
where both $\hat\omega_{i, j}$ and $\vec{d}$ are known constant vectors. We can use this to test if there exists $\Delta$. Combining the $k+\Delta k$ must maintain a new dot product with the last control, the same as the two controls. We can find the relation between $\Delta k$ and $\Delta$. Using the relation with the last control, we can solve for the relation between two weights for two $\grad_k$. Then, we can use the above equation to test if $\Delta$ exists. Furthermore, if such a direction exists, we can select the direction of $\Delta$ as the switching curve's gradient. 

Further, we can show that, inside a control synthesis region, there does not exist $\Delta k$ that satisfies the above constraint. One of the $\grad_k$ can produce the same dot product with one of the switching controls and the last control. The other switching control increases with this $\grad_k$, but not as much. Different $\grad_k$ cannot increase all three control moments on the switching curve.

\section{Proof of Concept}

Consider a simple Dubins' vehicle in 2D. We know the optimal solutions for the Dubins' vehicle and its control synthesis. We can use the known control synthesis to verify our necessary condition. Let us first consider the configuration slice of $q = (x, y, 0)$. For this case, let us first consider the switching curve between Left-Straight-Right (LSR), and the Right-Translation-Right (RSR), using the same notation from~\citet{bui1994shortest}. Most of the circle with center $(0, 1)$ is on this switching curve. Let us check the configurations $(1, 1, 0)$ and $(0, 1, 0)$ . The first configuration is on the switching curve, while the last one is not on the switching curve.  Because the switching curve is between LSR and RSR, the difference is the first control, so the switching curve is between the left rotation control and the right rotation control. At this configuration, both the left and right rotation, and the last control, which is also a right rotation, must all maintain a constant dot product of $1$ with $\vec{k}$. We can then derive $\vec{k} = (1, 1, \cdot)$. The $k$ satisfies all three controls. A $\Delta k$ direction does not exist that increases the dot product with all control moments. 

We can apply the same test to configuration $(0, 1, 0)$, which is not on the switching curve. We can again have the same gradient to $q$ for both controls. However, there does not exist a $k$ that satisfies the three control moments. So, the first test fails. Of course, this configuration and the system are somewhat special. Because for the Dubins vehicle, any configuration would admit a turn-drive-turn trajectory, there would always exist a valid $\vec{k}$ satisfying constraints for a turn-drive-turn path. These paths may not be the shortest but is always a feasible trajectory.  We then need to conduct further tests, to verify if a configuration is on the switching curve. Let us test the $\Delta k$ vector at these configurations. For example, at $(1, 1, 0)$, the two controls will lead to two different $\Delta k$, but only $(0, 1, 0)$ direction will lead to a feasible $\Delta$ for using Eq.~\eqref{eq:delta_relation}. The equation that constraints $\Delta$ becomes $(1+\|\Delta k\|)\Delta_x - \Delta_y = -\|\Delta k\|$, where $c$ depends on the scale of $\Delta k$. At this position, we know the gradient is following the circle, and if we let $\Delta_y = \sin\theta$ and $\Delta_x = 1-\cos\theta$, the constraint is satisfied for small $\theta$ and $\|\Delta k\|$.

\begin{figure}[b!]
    \centering
    \includegraphics[width=3in]{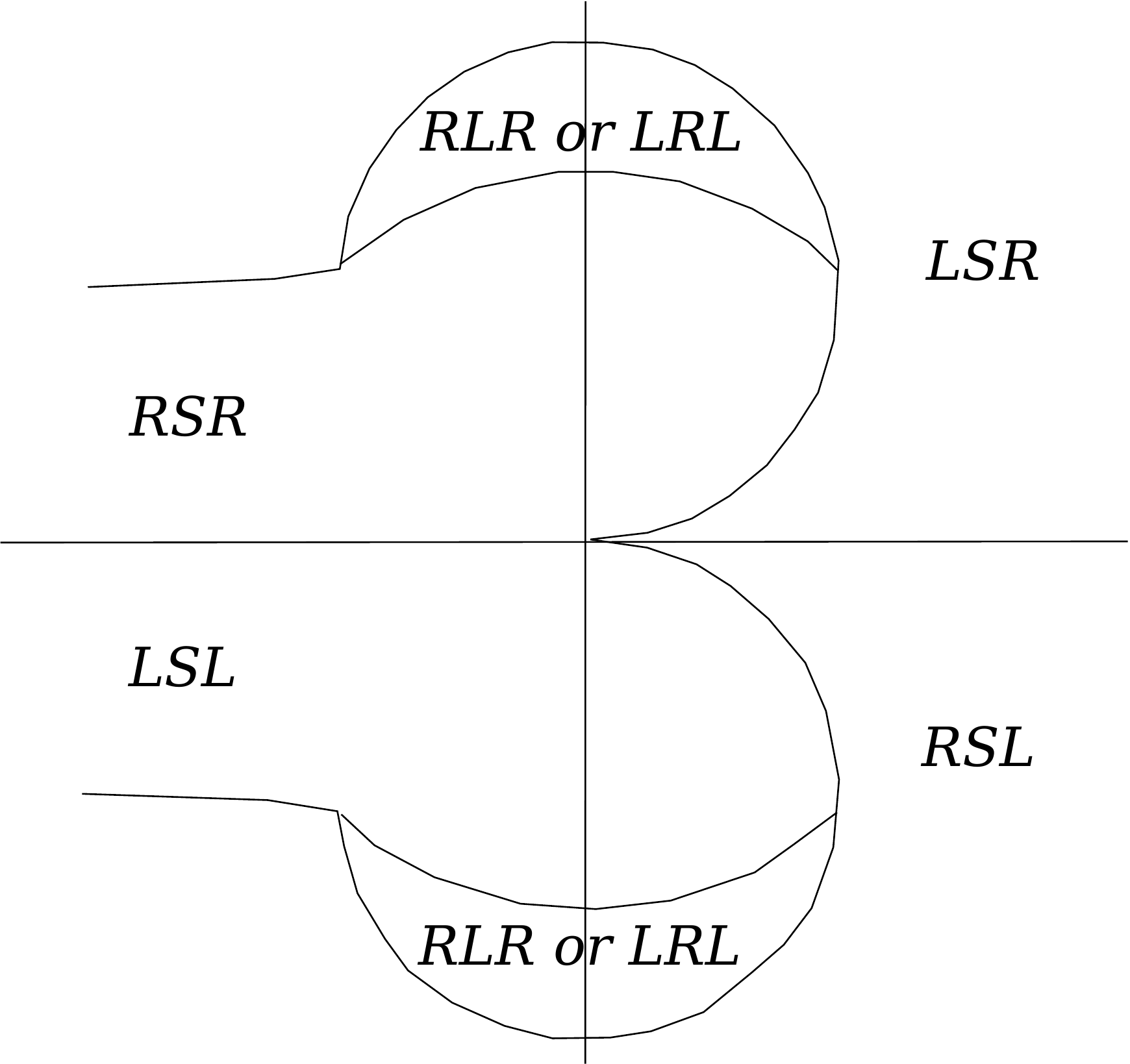}

\vspace{-0.2in}   
    
    \caption{The switching curve and control synthesis region for Dubins' vehicle, configuration slice $(x, y, 0)$, redrawn from~\citet{bui1994shortest}.}
    \label{fig:slice-0}\vspace{-0.3in}
\end{figure}

For another configuration, such as $(0.5, 0.4, 0)$ not on a switching curve. We can derive a feasible $k$ of $(1, 0.8, 0)$. The $\Delta k$ is either along $(1.4, -0.5)$ or along $(0.6, 0.5)$. Setting the $\Delta k = \alpha\cdot(1.4, -0.5)^T + \beta\cdot(0.6, 0.5)^T$, leading to $\beta = -\frac{7}{3}\alpha$. But $\grad_k = (0.6, 0.5)$ have a larger control moment with all three controls. The dot product with two right rotation control moments is equal. The dot product with the left rotation control moment is increasing but smaller than the other two. As shown in Figure~\ref{fig:slice-0}, the current configuration is inside the RSR region, and the two right rotation controls can maintain a constant dot product even after moving $k$ by $\Delta k$. In other words, there is better $k$ than the original $k$ for the two right rotations. Therefore, in this configuration, the left rotation should not be selected. The analysis also indicates that on the switching curve, there exists $\Delta k$ that can move all controls to an identical dot product after $\Delta$, and at the current location, the updated $k$ will not push all control moments to the same direction. This can also be verified at configuration $(1, 1, 0)$.


Combining the above results, we can now propose a new way to find configurations on the switching curve and finding corresponding control synthesis region.
\begin{enumerate}
    \item At a given configuration, we assume it is on a switching curve, between some controls $u_i$, and $u_j$, and the last control $u_k$.
    \item Test if there exist feasible $k$ that allow these three controls to satisfy $k\cdot(\hat\omega_i\times\vec{rg}) = 1$ for rotation, or $k\cdot v_i = 1$ for translation. 
    \item If no such $k$, it is not on a switching curve. If it satisfies, need further testing. 
    \item Find potential basis for $\Delta k$ using Eq.~\eqref{eq:grad_k}, for all controls. 
    \item Set $\Delta k = \sum_i \alpha_i \grad_k(u_i))$ where $\grad_k(u_i)$ is the gradient with respect to $k$ for $u_i$. 
    \item Test if there exist $\grad_k$ that increases the dot product with all control moments. 
    \item If yes, then not on the switching curve, but inside a control synthesis region. We can identify the  first and last control for the control synthesis region. Simulate and find the corresponding control synthesis region control sequence. Break.
    \item Potentially on a switching curve. Use $u_k$ to derive the relation among $\alpha_i$. 
    \item Bring $\Delta k$ into Eq.~\eqref{eq:joint_deltas}, find feasible $\Delta$ relation. 
    \item Find a $\Delta$ direction that admit same dot product with all three control moments, and go back to step 4. 
\end{enumerate}

There may also exist a new way to derive the shortest path and corresponding $k$. Following a similar procedure, assuming the configuration is on the switching curve. Then, we derive the $\Delta k$ and test if some control choices can be better than the others. If so, we can derive the new $k$ vector and find a better feasible path. If a feasible $k$ cannot be found at the beginning, we can just use two controls to derive an equation for $k$ to satisfy and then take the derivatives of $k$ to find the direction $\Delta k$ to increase the dot product with these controls. This procedure works well with Dubins' vehicle and Reeds-Shepp's car but needs further testing and proof of correctness.

\newpage

\section{Conclusion}

This work presents the necessary conditions for configurations inside and on the boundary of control synthesis regions. Using Pontraygon's Maximum Principle (PMP), we first show that a three-dimensional search is sufficient to find the shortest path for any given pair of start and goal. Then, we proved that the nearby configurations would have similar constant vectors $\vec{k}$ to maintain the constant dot product with the control moment. Further, we show the necessary conditions for a configuration to be on the switching curve. We show that the gradients of the constraints for two controls must agree for a configuration to be on a switching curve. We further derive a procedure to quickly walk towards a switching curve if the shortest path and the corresponding $\vec{k}$ are known for a given configuration. 

One of the main directions for future work is to prove the procedure's correctness for finding control synthesis and the shortest path presented at the end of the last section. The current procedure depends on knowing the shortest path between a start and the goal (origin), and the current shortest path algorithm is a three-dimensional search and can be slow. Another future work is to include more examples and experiments for different systems. 

The work verifies the necessary condition for the control synthesis boundaries. However, the details of generalizing such verification to a complete process to plot out the control synthesis region are not yet shown. More experiments and accompanying details of the approach will be the focus of the future work. 
\bibliographystyle{plainnat}
\bibliography{refs}

\end{document}